\documentclass[final,3p,times,twocolumn]{elsarticle}

\usepackage{amsmath,amssymb,amsthm}
\usepackage{graphicx}
\usepackage{grffile}
\usepackage{epstopdf}
\usepackage{multirow}

\usepackage{url}
\usepackage{textcomp}
\usepackage{epstopdf}
\usepackage{algorithm}
\usepackage{algpseudocode}
\usepackage{booktabs}
\usepackage{cases}

\usepackage{caption}
\usepackage{float}
\usepackage{threeparttable}

\newtheorem*{definition} {Definition} 
\newtheorem*{lemma} {Lemma}

\journal{Neurocomputing}

\begin{document}

\begin{frontmatter}

\title{Adaptive Weighted Nonnegative Matrix Factorization for Robust Feature Representation}

\author[neu]{Tingting Shen}
\author[neu]{Junhang Li}
\author[neu]{Can Tong}
\author[neu]{Qiang He}
\author[neu]{Chen Li}
\author[stevens]{Yudong Yao}
\author[neu]{Yueyang Teng \corref{C1}}
\cortext[C1]{Corresponding author.}
 \ead{tengyy@bmie.neu.edu.cn}

\address[neu]{College of Medicine and Biological Information Engineering, Northeastern University, Shenyang 110169, China}
\address[stevens]{Department of Electrical and Computer Engineering, Stevens Institute of Technology,
 Hoboken, NJ 07030, USA}

\begin{abstract}
Nonnegative matrix factorization (NMF) has been widely used to dimensionality reduction in machine learning. However, the traditional NMF does not properly handle outliers, so that it is sensitive to noise. In order to improve the robustness of NMF, this paper proposes an adaptive weighted NMF, which introduces weights to emphasize the different importance of each data point, thus the algorithmic sensitivity to noisy data is decreased. It is very different from the existing robust NMFs that use a slow growth similarity measure. Specifically, two strategies are proposed to achieve this: fuzzier weighted technique and entropy weighted regularized technique, and both of them lead to an iterative solution with a simple form. Experimental results showed that new methods have more robust feature representation on several real datasets with noise than exsiting methods.
We make our code available at \url{https://github.com/WRNMF/RobustNMF.git}.
\end{abstract}


\end{frontmatter}


\section{Introduction}
With the rapid development of computer science, we are acculating more and more large-scale and high-dimensional data, thereby bringing about a new challenge for data analysis. It is very difficult to analyze many high-dimensional and high-throughout datasets. Consequently, dimensionality reduction has become an essential step. There are many methods to solve this problem including vector quantization (VQ) \cite{gersho_vector_1992}, singular value decomposition (SVD) \cite{hu_singular_2017}, principal component analysis (PCA) \cite{abdi_principal_2010}, independent component analysis (ICA) \cite{hyvarinen_independent_2000} and nonnegative matrix factorization (NMF) \cite{Lee1999Learning}, etc. In these methods, NMF is attracting more and more attention, which has been widely used in many fields, such as pattern recognition \cite{gong_ficks_2015}, image processing \cite{li_robust_2018} and  data mining \cite{tepper_compressed_2016}.\\

The traditional NMF decomposes a high-dimensional nonnegative matrix into the product of two low-dimensional nonnegative matrices that are respectively called base and representation \cite{Lee1999Learning}. In recent decades, the theoretical research and practical applications of NMF are being carried out intensively, and many variants are developed for different tasks. For example,  Ding \emph{et al.} \cite{ding_convex_2010} proposed SemiNMF by easing the nonnegative constrains on original and base matrices to expand the application scope of NMF. They also limited base matrix to convex combinations of data points for developing the ConvexNMF method.
Guillamet \emph{et al.} \cite{guillamet2003introducing} proposed the weighted NMF method (WNMF), which introduces to improve the NMF capabilities of some representing positive local data.

Most of the NMF variants use the square of the Euclidean distance  or Kullback-Leibler (KL) divergence to measure the similarity between the high-dimensional nonnegative matrix and the product of two low-dimensional nonnegative matrices. However, the results generally cause sensitivity to outliers. In order to improve the robustness of NMF, Xing \emph{et al.} \cite{xing_nonnegative_2018} designed 
the NMFL2 method that used the Euclidean distance to take the place of the square of the Euclidean distance, so that the importance of the outliers decreased. Du \emph{et al.} \cite{du_robust_2012} took the Huber function to measure the quality approximation by considering its connection to $L_2$-norm and $L_1$-norm.
As can be seen, these NMF methods use some similarity measures insensitive to outliers for a robust feature representation.

This paper proposes a totally different robust NMF type, which assigns an adaptive weight for each data point to represent its outlier degree. Certainly, an outlier should be paid less attention than a normal point. The task is implemented by two methods: fuzzier weighted technique and entropy weighted regularized technique, in which one utilizes a power hyperparameter to control the weights distribution and the other uses the entropy to regularize the weights. Then, they are solved by the Lagrange multiplier method for obtaining two solutions with a simple form. Experiments on several datasets with synthetic noise displayed that the proposed methods always achieve a better performance than some existing methods.

\section{Methodology}
$\bf{Notation}$: In this paper, matrices are denoted as capital letters. For a matrix $A$, the $(i, j)-th$ element is indicated as $A_{ij}$; $A^T$ denotes the transpose of $A$; the Frobenius norm is represented as $||\cdot||_F $; the symbols $\odot$ and $\oslash$ mean the item-by-item multiplication and division of two matrices, respectively; $A \geq 0$ means that all elements of $A$ are equal to or larger than 0.

\subsection{Related works}
NMF decomposes a nonnegative matrix $X\in R^{M \times N}$ into the product of two low-dimensional nonnegative matrices: one is the base matrix $W\in R^{M \times L}$ and the other is the representation matrix $H\in  R^{L\times N}$, where  $L<<min\{M,N\}$. It can be formulated as:
\begin{equation}
	X = W H
	\label{eq1}
\end{equation}

As well known, it is difficult to be solved, even there may be no solution, so we turn to an optimization method for achieving the best approximation.
There are many criteria to measure the similarity between $X$ and $WH$. In general, the square of the Frobenius norm and the Kullback-Leibler (KL) divergence are used. In this paper, the square of the Frobenius norm is adopted and the formula is expressed as:
\begin{eqnarray}
		\min&&~F_1(W,H)=\left\|X-WH\right\|_F^2\nonumber \\
		s.~t.&&~W\geq 0,~H \geq 0
		\label{eq2}
\end{eqnarray}

Although the objective function is convex quadratic function on either $W$ or $H$, it isn't convex for both of them.
Thus, it will be alternatively minimized with regard to $W$ and $H$, in which the constructed auxiliary function will be important to derive the iterative update rule.

\begin{definition} (Auxiliary function) If the function $G(h,h')$ satisfies the following conditions:
\begin{eqnarray}
	&&G(h,h')\geq F(h)\\
	&&G(h',h')=F(h')
	\label{eq3}  
\end{eqnarray}
where $h'$ is a given value; then, $G(h,h')$ is the auxiliary function of $F(h)$ on $h'$. 
\end{definition}
Then, we easily draw the following conclusion:
\begin{lemma}If $G(h,h')$ is an auxiliary function of $F(h)$, then under the update rule:
\begin{equation}
	h^*=arg\mathop{\min}_{h}G(h,h')
	\label{eq4}
\end{equation}
the function $F(h)$ does not increase. 
\end{lemma}
\begin{proof}
The conditions satisfied by the auxiliary function make this proof marked because:
\begin{equation}
	F(h^{*})\leq G(h^{*},h')\leq G(h',h')\leq F(h')
	\label{eq5}
\end{equation}
\end{proof}

It can be observed that the original function should decrease if the auxiliary function reaches the minimum. Now, we drive the update rule for NMF. Consider $W$ first, where $W^t>0$ and $H>0$ are fixed. Let $\xi_{ijk}=W^t_{ik}H_{kj}/(W^tH)_{ij}$, certainly, $\xi_{ijk}\geq0$ and $\sum_{k=1}^L\xi_{ijk}=1$. Therefore, the auxiliary function is as below:
\begin{equation}
	f_1(W,W^t)=\sum_{i=1}^M\sum_{j=1}^N\sum_{k=1}^L\xi_{ijk}(X_{ij}-\frac{W_{ik}H_{kj}}{\xi_{ijk}})^2
	\label{eq6}
\end{equation}
The function is separable and its minimization is equivalent to the solutions of some univariate optimizaiton problems. Take the partial derivative of Eq. (\ref{eq6}) and set it to zero so that the following update rule is got.
\begin{equation}
	W \leftarrow W \odot (XH^T) ~\oslash~ (WHH^T)
	\label{eq7}
\end{equation}
Similarly, the update of $H$ is obtained:
\begin{equation}
	H \leftarrow H \odot (W^TX) ~\oslash~ (W^TWH)
	\label{eq8}
\end{equation}

\subsection{Proposed method}
Different from the previous methods that use a similarity measure with slow growth for outliers, we introduce an optimizable weight $Q_j$ to emphasize the importance of each data point, in which an outlier will be paid less attention than a normal point. 
\begin{eqnarray}
		\min&&~F_2(W,H,Q)=\sum_{i=1}^M \sum_{j=1}^NQ_{j}[X_{ij}-(WH)_{ij}]^2\nonumber\\ 
		s.~t.&&~ W\geq0,~H\geq0,~Q\geq0,~\sum_{j=1}^N Q_{j}=1
		\label{eq9}
\end{eqnarray}

We first consider the solution to $Q$ in an alternative optimization manner. Obviously, for fixed $W$ and $H$, $Q_{j}$ can be solved as $Q_{j}=1$ if $Z_{j}=min\{Z_{1}, Z_{2},..., Z_{N}\}$, otherwise 0, where $Z_j=\sum_{i=1}^M(X-WH)_{ij}^2$. The derivation is very easy, which is similar to that of the membership degree in k-means \cite{macqueen_methods_nodate}. It demonstrates the simple fact that only one $Q_j$ is 1, and the others are 0, i.e., the only one data point is involved in NMF. Such a phenomenon is certainly incompatible with the real problem. In fact, we generally think that the weights are in the range of $[0, 1]$ instead of 0 or 1, which 
is more consistent fuzzy logic thinking process. To address this issue, we propose two methods as follows.

\subsubsection{Fuzzier weighted robust NMF (FWRNMF)}
We give a hyperparameter $p>1$, representing fuzzy level, to smooth $Q_{j}$ so that weights fall within [0, 1]. The specific model is as follows:
\begin{eqnarray}
	\min&&~F_3(W,H,Q)=\sum_{i=1}^M \sum_{j=1}^NQ_{j}^p[X_{ij}-(WH)_{ij}]^2\nonumber\\ 
	s.~t.&&~ W\geq0,~H\geq0,~Q\ge0,~\sum_{j=1}^NQ_{j}=1
	\label{eq10}
\end{eqnarray}
We still alternatively minimize the objective function with respect to $Q$, $W$ and $H$. First, $Q$ is solved.
We need to construct the Lagrange function as:
\begin{equation}
	L(Q,\lambda)=\sum_{i=1}^M \sum_{j=1}^NQ_{j}^p[X_{ij}-(WH)_{ij}]^2-\lambda(\sum_{j=1}^NQ_{j}-1)
	\label{eq11}
\end{equation}
where $\lambda$ is the Lagrange multiplier.

By setting the gradient of Eq. (\ref{eq11}) with respect to $\lambda$ and $Q_{j}$ to zero, we obtain the following equations system:
\begin{numcases}{}
	\frac{\partial L}{\partial \lambda}&=$\sum_{j=1}^NQ_{j}-1=0$ \label{eqsystem1}\\
	\frac{\partial L}{\partial Q_{j}}&=$pQ_{j}^{p-1}\sum_{i=1}^M[X_{ij}-(WH)_{ij}]^2-\lambda=0$
\end{numcases}
From Eq. (14), we know that:
\begin{equation}
	\begin{aligned}
		Q_{j}=\sqrt[\leftroot{-2}\uproot{12}p-1]{\frac{\lambda}{p}}\sqrt[\leftroot{-2}\uproot{14}p-1]{\frac{1}{\sum_{i=1}^M[X_{ij}-(WH)_{ij}]^2}}
		\label{eq14}
	\end{aligned}
\end{equation}
Substituting Eq. (\ref{eq14}) into Eq. (13), we have:
\begin{equation}
	\begin{aligned}
		\sqrt[\leftroot{-2}\uproot{14}p-1]{\frac{\lambda}{p}}=\frac{1}{\sum_{j=1}^N\sqrt[\leftroot{-2}\uproot{8}p-1]{\frac{1}{\sum_{i=1}^M[X_{ij}-(WH)_{ij}]^2}}}
		\label{eq15}
	\end{aligned}
\end{equation}
Substituting it into Eq. (\ref{eq14}), we achieve that\\
\begin{equation}
	\begin{aligned}
		Q_{j}=\frac{\sqrt[\leftroot{-2}\uproot{8}p-1]{\frac{1}{\sum_{i=1}^M[X_{ij}-(WH)_{ij}]^2}}}{\sum_{l=1}^N\sqrt[\leftroot{-2}\uproot{8}p-1]{\frac{1}{\sum_{i=1}^M[X_{il}-(WH)_{il}]^2}}}
		\label{eq16}
	\end{aligned}
\end{equation}

Then, we can solve $W$ and $H$ with fixed $Q$, which is similar to the traditional NMF method. For example, we can construct the following auxiliary function about $W$:
\begin{equation}
	\begin{aligned}
		f_2(W,W^t)=&\sum_{i=1}^M \sum_{j=1}^N\sum_{k=1}^LQ_{j}\xi_{ijk}(X_{ij}-\frac{W_{ik}H_{kj}}{\xi_{ijk}})^2
		\label{eq17}
	\end{aligned}
\end{equation}
Setting the partial derivative of $f_2(W,W^t)$ to zero yields the following update rule:
\begin{equation}
	W \leftarrow W\odot(XQ^pH^T)~\oslash~(WHQ^p H^T)
	\label{eq18}
\end{equation}
where $Q^p=diag([Q_1^p, Q_2^p,...,Q_N^p])$. 

Similarly, we can also easily obtain the update rule for $H$ as follows:
\begin{equation}
	H \leftarrow H \odot (W^TX)\oslash(W^TWH)
	\label{eq19}
\end{equation}

The update rules to $W$ and $H$ are similar to the existing WNMF methods in \cite{belhumeur_eigenfaces_nodate,kuang_symnmf_2015}. Optimizing FWRNMF is summarized as follows in $\mathbf{Algorithm~ 1}$:
\begin{algorithm}[htb]
	\caption{FWRNMF}
	\label{alg:Framwork}
	\begin{algorithmic}[1]
		\Require
		Given the input nonnegative matrix $X$, reduced dimension number $K$ and hyperparameter $p$;
		\Ensure
		the weight matrix $Q$, the base matrix $W$ and the representation matrix $H$;
		\State Randomly initialize $W> 0$ and $H>0$;
		\While{not convergence}
		\label{code:fram:trainbase}
		\State Update $Q$ by Eq. (\ref{eq16});
		\label{code:fram:add}
		\State Update $W$ by Eq. (\ref{eq18});
		\label{code:fram:classify}
		\State Update $H$ by Eq. (\ref{eq19});
		\label{code:fram:classify}
		\EndWhile
		\label{code:fram:select} \\
		\Return  $Q$, $W$ and $H$.
	\end{algorithmic}
\end{algorithm}

\subsubsection{Entropy weighted robust NMF (EWRNMF)}
We apply an information entropy to regularize the objective function of NMF to obtain the weights in the range of [0, 1] instead of 0 or 1. The information entropy can indicate the uncertainty of weights.
\begin{eqnarray}
	\min&&~F_4(W,H,Q)=\sum_{i=1}^M \sum_{j=1}^NQ_{j}[X_{ij}-(WH)_{ij}]^2\nonumber\\
	&&~~~~~~~~~~~~~~~~~~~+\gamma\sum_{j=1}^NQ_{j}ln(Q_{j})\nonumber\\ 
	s.~t.&&~ W\geq0,~H\geq0,~Q\ge0,~\sum_{j=1}^NQ_{j}=1
	\label{eq20}
\end{eqnarray}
where $\gamma>0$ is a given hyperparameter to smooth the weights. The first term of the objective function is the sum of errors, and the second term is the negative entropy of the weights. The original objective function in Eq. (\ref{eq9}) results in only one data point being involved in feature representation, and the entropy regularizer will stimulate more points to participate in feature representation. \\

We construct the Lagrange function of Eq. (\ref{eq20}) with respect to $Q$ as:
\begin{eqnarray}
	L(Q,\lambda)&&=\sum_{i=1}^M \sum_{j=1}^NQ_{j}[X_{ij}-(WH)_{ij}]^2\nonumber\\
	&&+\gamma\sum_{j=1}^NQ_{j}ln(Q_{j})-\lambda(\sum_{j=1}^NQ_{j}-1)
	\label{eq21}
\end{eqnarray}
where $\lambda$ is still the Lagrange multiplier.

By setting the gradient of Eq. (\ref{eq21}) with respect to $\lambda$ and $Q_{j}$ to zero, we obtain the following equations system:\\
\begin{numcases}{}
	\frac{\partial L}{\partial \lambda}&=$\sum_{j=1}^NQ_{j}-1=0$ \label{eqsystem1}\\
	\frac{\partial L}{\partial Q_{j}}&=$\sum_{i=1}^M[X_{ij}-(WH)_{ij}]^2+\gamma lnQ_{j}+\gamma-\lambda=0$
\end{numcases}
From Eq. (24), we know that:\\
\begin{equation}
	\begin{aligned}
		Q_{j}=e^{\frac{\lambda-\gamma}{\gamma}}e^{-\frac{\sum_{i=1}^M[X_{ij}-(WH)_{ij}]^2}{\gamma}}
		\label{eq24}
	\end{aligned}
\end{equation}
Substituting Eq. (\ref{eq24}) into Eq. (23), we have:
\begin{eqnarray}
	e^{\frac{\lambda-\gamma}{\gamma}} =\frac{1}{\sum_{j=1}^N e^{{-\frac{\sum_{i=1}^M[X_{ij}-(WH)_{ij}]^2}{\gamma}}}}  
	\label{eq25}
\end{eqnarray}
Substituting this expression to Eq. (\ref{eq24}), we find that:
\begin{equation}
	Q_{j}=\frac{e^{-\frac{\sum_{i=1}^M[X_{ij}-(WH)_{ij}]^2}{\gamma}}}{\sum_{l=1}^N e^{-\frac{\sum_{i=1}^M[X_{il}-(WH)_{il}]^2}{\gamma}}}  
	\label{eq26}
\end{equation}

 Then, we can solve $W$ and $H$ with fixed $Q$, which is similar to FWRNMF. 
\begin{eqnarray}
	W &\leftarrow& W\odot (XQH^T)~\oslash~(WHQ H^T)    \label{eq28}\\
	H &\leftarrow& H \odot (W^TX)\oslash(W^TWH)
	\label{eq29}
\end{eqnarray}

EWRNMF is summarized as follows in $\mathbf{Algorithm~ 2}$:
\begin{algorithm}[htb]
	\caption{ EWRNMF}
	\label{alg:Framwork}
	\begin{algorithmic}[1]
		\Require
		Given the input nonnegative matrix $X$, reduced dimension number $K$ and hyperparameter $\gamma$;
		\Ensure
		the weight matrix $Q$, the base matrix $W$ and the representation matrix $H$;
		\State Randomly initialize $W> 0$ and $H> 0$;
		\While{not convergence}
		\label{code:fram:trainbase}
		\State Update $Q$ by Eq. (\ref{eq26});
		\label{code:fram:add}
		\State Update $W$ by Eq. (\ref{eq28});
		\label{code:fram:classify}
		\State Update $H$ by Eq. (\ref{eq29});
		\label{code:fram:classify}
		\EndWhile
		\label{code:fram:select} \\
		\Return  $Q$, $W$ and $H$.
	\end{algorithmic}
\end{algorithm}

\subsection{Extensions}
Obviously, the proposed techniques are irrelative to the measure between $X$ and $WH$, so it can easily be popularized to KL-divergence, $\alpha$-divergence, Orthogonal NMF, ConvexNMF, etc, which demonstrates its good compatibility. We omit the derivation processing since it is very similar to the proposed ones.

\section{Experiments}
\subsection{Experimental description}
Experiments were performed on an HP Compaq PC with a 2.90-GHz Core i7-10700 CPU and 16 GB memory, and all the methods were implemented in MATLAB. We compared the performance of the two proposed methods with the traditional NMF (named EucNMF), ConvexNMF, NMFL2 and HuberRNMF on nine public datasets, including the Yale, ORL, GTFD, Winequality-red, Page-blocks, Balance, Thyroid, WDBC and Dimdata. The important details of these datasets are shown in Table \ref{tab1}. We add Gaussian noise to the data to interpret the robustness of the methods as below:
\begin{equation}
	X_{ij} = X_{ij} + cN(0,X_{ij})\label{eq:noise}
\end{equation}
where $c$ control the noise level and $N(0,X_{ij})$ follows a Gaussian distribution with mean 0 and variance $X_{ij}$. It implies that noise relates to the scale of data, and a bigger value corresponds to a higher noise, which is reasonable in practical applications.

\begin{table}[!h] \centering
	\begin{threeparttable}
		\caption{Description of the datasets.}
		\label{tab1}
		\begin{tabular}{lcccccc}  
			\toprule   
			Dataset & Samples & Dimensions & Classes \\	
			\midrule   
			Yale \cite{belhumeur_eigenfaces_nodate}&165&$32\times32$&15   \\  
			ORL \cite{kuang_symnmf_2015}&400&$32\times32$&40\\
			GTFD\tnote{1} &750&$32\times32$&50\\
			Winequality-red\tnote{2}  &4898  &12  &11\\
			Page-blocks\tnote{2}   &5473  &10   &5\\
			Balance\tnote{2}    &625 &4 &3\\
			Thyroid\tnote{2}      &7200 &21  &3\\
			WDBC\tnote{2}      &569   &32    &2\\
			Dimdata\tnote{3} &4192&14&2\\
			\bottomrule
		\end{tabular}
		\begin{tablenotes}
			\item[1] {ftp://ftp.ee.gatech.edu/pub/users/hayes/facedb/}
			\item[2] {http://archive.ics.uci.edu/ml/datasets/}
			\item[3] {http://pages.cs.wisc.edu/~olvi/}
		\end{tablenotes}
	\end{threeparttable}
\end{table}

After obtaining a new feature representation, we use k-means to cluster them and then evaluate the clustering results. Clustering accuracy (ACC) \cite{liu_constrained_2012}, \cite{plummer1986matching} and normalized mutual information (NMI) \cite{cai_document_nodate}, \cite{shahnaz_document_2006} are used to evaluate the performance of these clustering results.

Given a set of the ground true class labels $y$ and the obtained cluster labels $y'$, the clustering accuracy is defined as:
\begin{equation}
	ACC=\frac{\sum_{i=1}^N\delta(y_i,map(y'_i))}{N}
	\label{eq34}
\end{equation}
where:
$$\delta(a,b)=
\begin{cases}
	1,& \text{\emph{a = b}}\\
	0,& \text{otherwise}
\end{cases}$$
and $map(\cdot)$ is a permutation mapping function that maps the obtained cluster labels to the real labels. The higher the ACC value is, the better the clustering performance.

NMI is used to calculate the agreement between the two data distributions and is defined as follows:
\begin{equation}
	NMI(y,y')=\frac{MI(y,y')}{max(H(y),H(y'))}
	\label{eq35}
\end{equation}
where $H(y)$ is the entropy of $y$. $MI(y,y')$ quantifies the amount of information between two random variables (i.e., $y$ and $y'$) and is defined as:
\begin{equation}
	MI(y,y')=\sum_{y_i\in y,y'_j\in y'}p(y_i,y'_j)log(\frac{p(y_i,y'_j)}{p(y_i)p(y'_j)})
	\label{eq36}
\end{equation}
where $p(y_i)$ and $p(y'_j)$ are the probabilities that a data point selected from the dataset belongs to the clusters $y_i$ and $y'_j$, respectively; and $p(y_i,y'_j)$ is the joint probability that an arbitrarily selected data point belongs to clusters $y_i$ and $y'_j$ concurrently. NMI ranges from 0 to 1, and the larger NMI is, the better the clustering performance.

Because NMF does not have a sole solution, we randomly initialized 10 times to obtain an averaged ACC and NMI for a creditable comparsion.
And we force reduced dimensionalty to equal to clustering number, which is a usual practice.

\subsection{Clustering performance}
We first compare all the methods on all the datasets, in which the hyperparameters of HuberRNMF, FWRNMF and EWRNMF are optimally from $cutoff\in\left\{10^i,i=-4,-3,...,3,4\right\}$, $p\in\left\{1.5,2,...,10.5,11\right\}$, and $\gamma\in\left\{10^i,i=-4,-3,...,3,4\right\}$, respectively. The noise level is set to be $c=0.05$. In the following experiments, we always use such a noise level unless otherwise specified. The optimal hyperparameter statergy is also used in the following experiments unless otherwise stated. Clustering results are shown in Tables \ref{tab2} and \ref{tab3}. We can see that either FWRNMF or EWRNMF always exhibits better performance than other methods, meanwhile NMF2L and HuberRNMF are generally better than the traditional NMF and ConvexNMF. It demonstrates the feasibility and effectivity of the proposed methods. In general, none of FWRNMF or EWRNMF has established total supremacy over each other.

\begin{table*}[!h] 
	\renewcommand\arraystretch{1.2}
	\caption{ACC on the databases (\%).}
	\label{tab2}	\centering
	\begin{tabular}{lccccccc}  
		\toprule   
		Method & EucNMF & ConvexNMF  & NMFL2 & HuberRNMF & FWRNMF & EWRNMF \\  
		\midrule   
		Yale  & 39.82 &  30.76  & 40.91 & 41.06 & \bf{41.91} & 41.61   \\  
		ORL  & 66.07 &  26.06  & 66.35 & 67.99 & \bf{69.02} & 67.92   \\   
		GTFD  & 62.56 &  49.32  & 63.76 & 64.73 & \bf{66.85} & 64.17   \\ 
		Winequality-red  &  30.29 &  28.99  & 29.66 & 30.44 & \bf{42.40} & \bf{42.40}  \\
		Page-blocks  &34.17 &  34.99  & 38.71 & 36.67 & \bf{89.71} & \bf{89.71}  \\ 
		Balance  & 51.70 &  51.34  & 50.78 & 51.93 & 51.66 & \bf{54.55}  \\
		Thyroid  & 54.04 &  55.83  & 54.65 & 54.19 & \bf{92.56} & \bf{92.56}  \\ 
		WDBC  & 87.69 &  83.68  & 87.70 & 88.58 & 89.01 & \bf{89.69}   \\ 
		Dimdata  & 52.45 &  54.26  & 52.32 & 52.45 & \bf{69.74} & 67.42    \\ 
		
		\bottomrule  
	\end{tabular}
\end{table*}

\begin{table*}[!h] 
	\renewcommand\arraystretch{1.2}
	\caption{NMI on the databases (\%).}
	\label{tab3}	\centering
	\begin{tabular}{lccccccc}  
		\toprule   
		Method & EucNMF & ConvexNMF  & NMFL2 & HuberRNMF & FWRNMF & EWRNMF \\  
		\midrule   
		Yale  &  45.23 & 38.04 & 46.20 & 46.31 & 47.05 & \bf{47.15}    \\  
		ORL  &  83.48 & 50.87  & 83.66 & 84.40 & \bf{84.84} & 84.34   \\   
		GTFD  & 85.48 & 73.93  & 85.60 & 86.05 & \bf{86.71} & 85.90    \\ 
		Winequality-red  &   10.56 & 8.47  & 10.53 & 10.63 & 9.88 & \bf{10.79}  \\
		Page-blocks  &12.16 & 6.78  & 13.27 & 12.57 & \bf{15.33} & 12.55  \\  
		Balance  &  10.78 & 13.01  & 10.11 & 11.73 & 12.84 & \bf{17.48}   \\
		Thyroid  &  1.60 & 1.55  & 1.76 & 1.61 & \bf{3.61} & 2.41   \\   
		WDBC  &  51.18 & 34.76  & 51.20 & 54.32 & 54.54 & \bf{54.57}    \\ 
		Dimdata  & 0.18 & 0.84  & 0.16 & 0.18 & \bf{16.54} & 14.33     \\ 
		\bottomrule  
	\end{tabular}
\end{table*}

\subsection{Hyperparameter selection}
We inspect the ability of FWRNMF and EWRNMF to improve the performance of the traditional NMF on the Yale and WDBC datasets for dimensionality reduction and k-means clustering. The hyperparameters of FWRNMF and EWRNMF are still selected optimally as above. 
As shown in Figures \ref{fig1} and \ref{fig2}, regardless of the hyperparameter $p$ or $\gamma$, the proposed methods indeed provide better performance than the traditional NMF and can achieve a consistently superior clustering result on extensive hyperparameters. 

\begin{figure*}[!h]
	\centering
	\includegraphics[height=165pt,width=400pt]{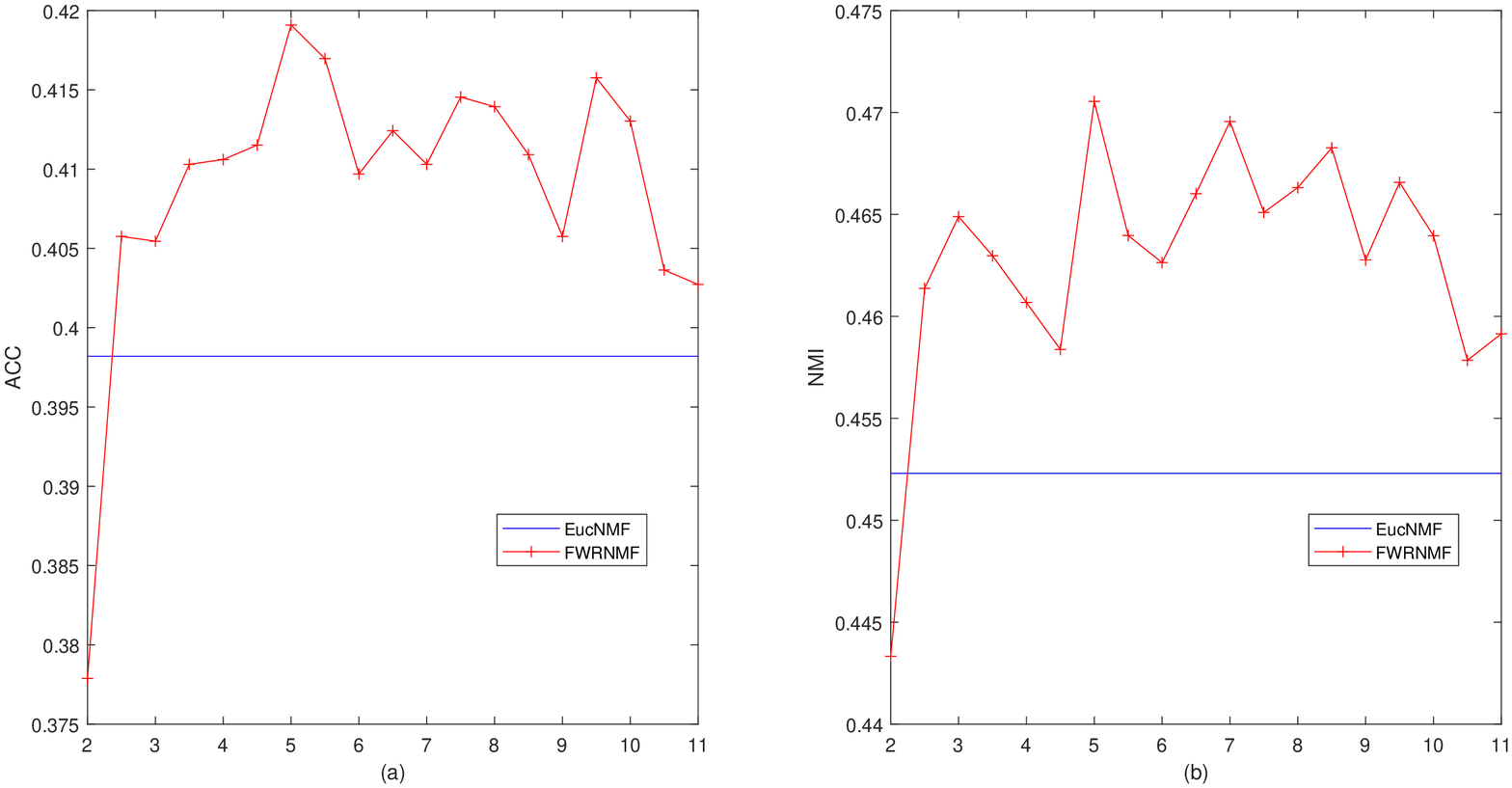}\label{4}
	\caption{Clustering performance versus the hyperparameter $p$ on the Yale dataset: (a) ACC and (b) NMI.}
	\label{fig1}
\end{figure*}

\begin{figure*}[!h]
	\centering
	\includegraphics[height=165pt,width=400pt]{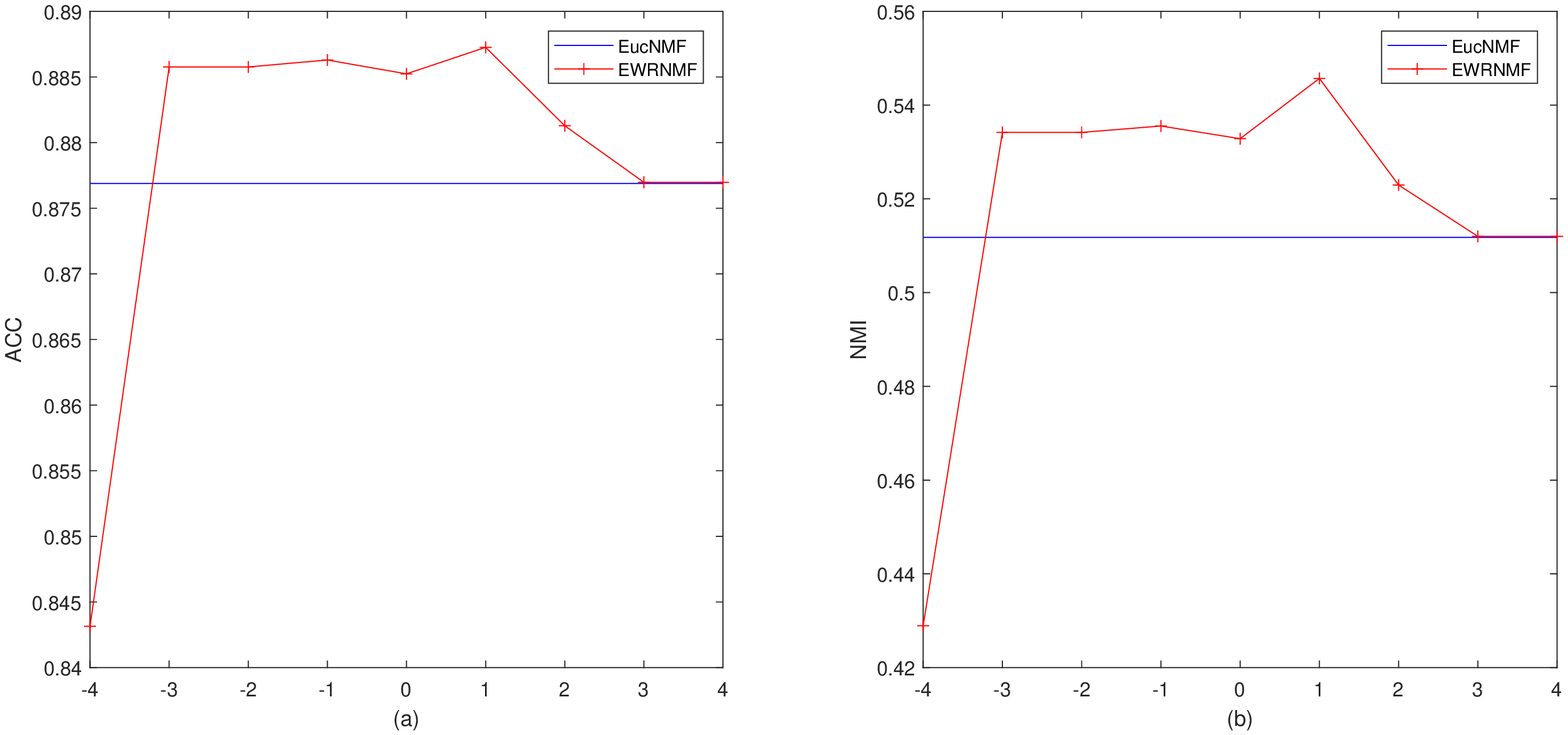}\label{4}
	\caption{Clustering performance versus the hyperparameter $\gamma$ (the abscissa is the value of i $\in\left\{\gamma=10^i,i=-4,-3,...,3,4\right\}$) on the WDBC dataset: (a) ACC and (b) NMI.}
	\label{fig2}
\end{figure*}

\subsection{Clustering results versus noise level}
We investigate the robustness of the proposed methods with regard to different noise levels on two datasets: GTFD and Balance, which add Gaussian noise to each dataset by Eq. \eqref{eq:noise} with the noise level $c$ ranging from 0.02 to 0.1. The results are shown in Figures \ref{fig:noise_GTFD} and \ref{fig:noise_Balance}. It can be seen that the effect of FWRNMF is particularly prominent in GTFD and EWRNMF performs well in Balance. Furthermore, regardless of which dataset, either of two proposed methods always obtains the best results, which shows a strong robustness.

\begin{figure*}[htbp]
	\centering
	\includegraphics[height=165pt,width=400pt]{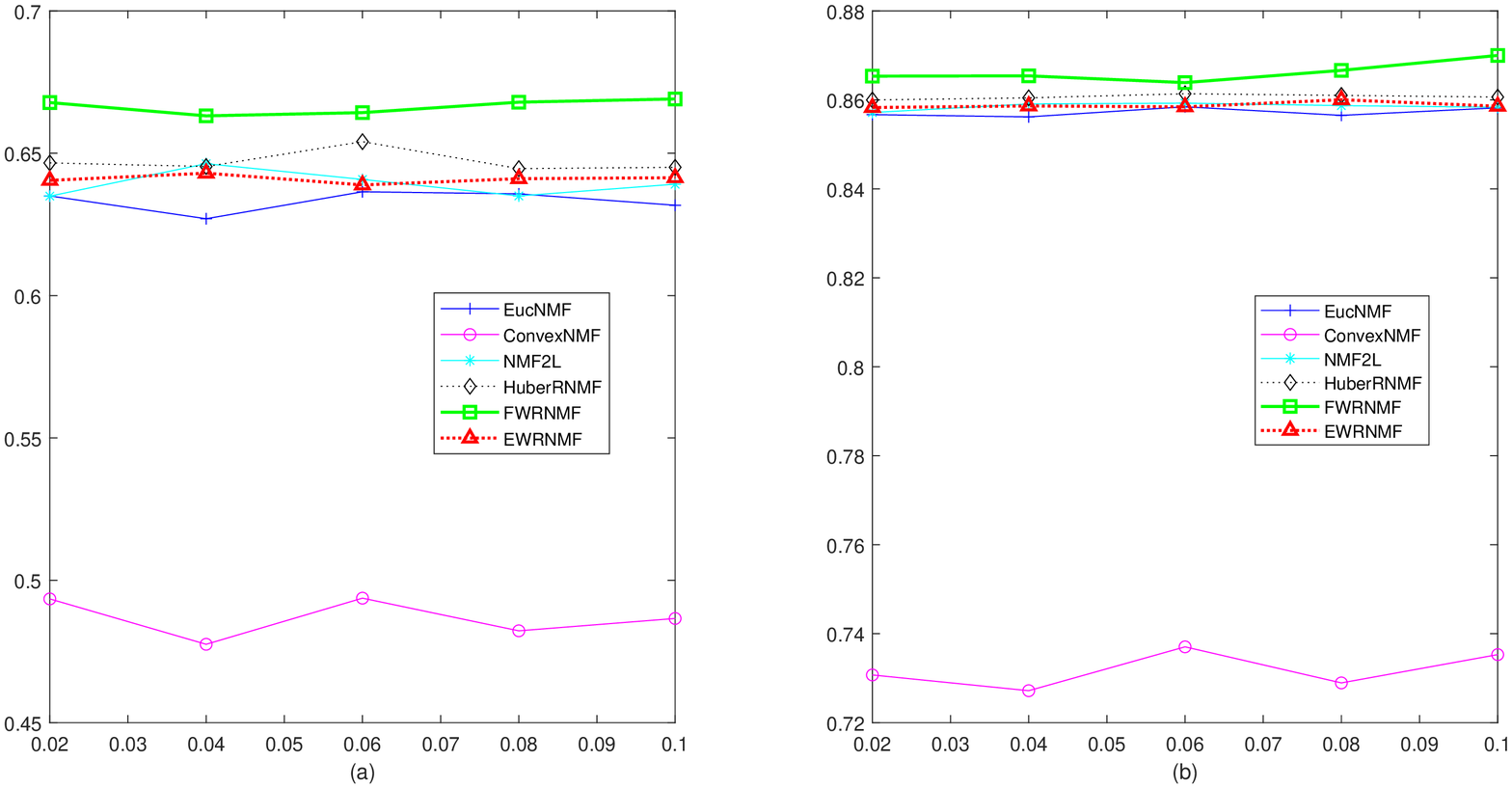}
	\caption{Clustering performance versus noise level  on the GTFD dataset: (a) ACC and (b) NMI.}
	\label{fig:noise_GTFD}
\end{figure*}

\begin{figure*}[htbp]
	\centering
	\includegraphics[height=165pt,width=400pt]{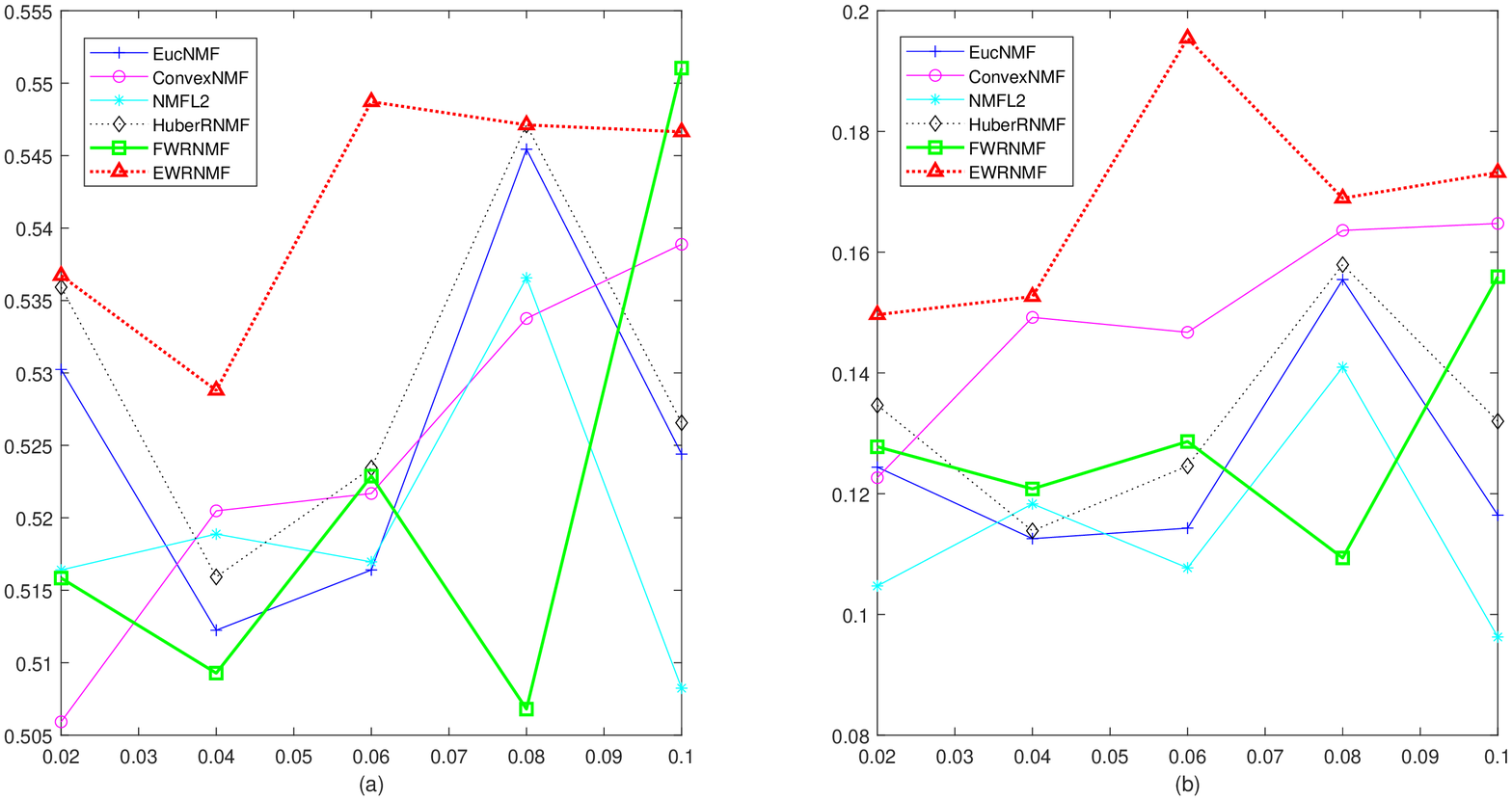}
	\caption{Clustering performance versus noise level on the Balance dataset: (a) ACC and (b) NMI.}
	\label{fig:noise_Balance}
\end{figure*}


\subsection{Clustering results versus cluster number}
We study the relationship between the evaluation standards and cluster number on the Yale and GTFD datasets, where cluster number ranges from 2 to 10 are selected. The experimental details are described as follows:\\
\indent 1) Randomly select \emph{k} categories as a subset for the following experiment.\\
\indent 2) Randomly initialize $W$, $H$ and obtain new representations, then cluster them by k-means. Hyperparameters are selected according to the above instruction.\\
\indent 3) Repeat 1) and 2) 10 times to obtain an average result.

It can be seen in Figures \ref{fig6} and \ref{fig7} that either of our methods is better than the other methods. In detail, for the Yale dataset, EWRNMF is the best, for the GTFD dataset, FWRNMF is the best; however, none of the proposed methods is always the best.


\begin{figure*}[htbp]
	\centering
	\includegraphics[height=165pt,width=400pt]{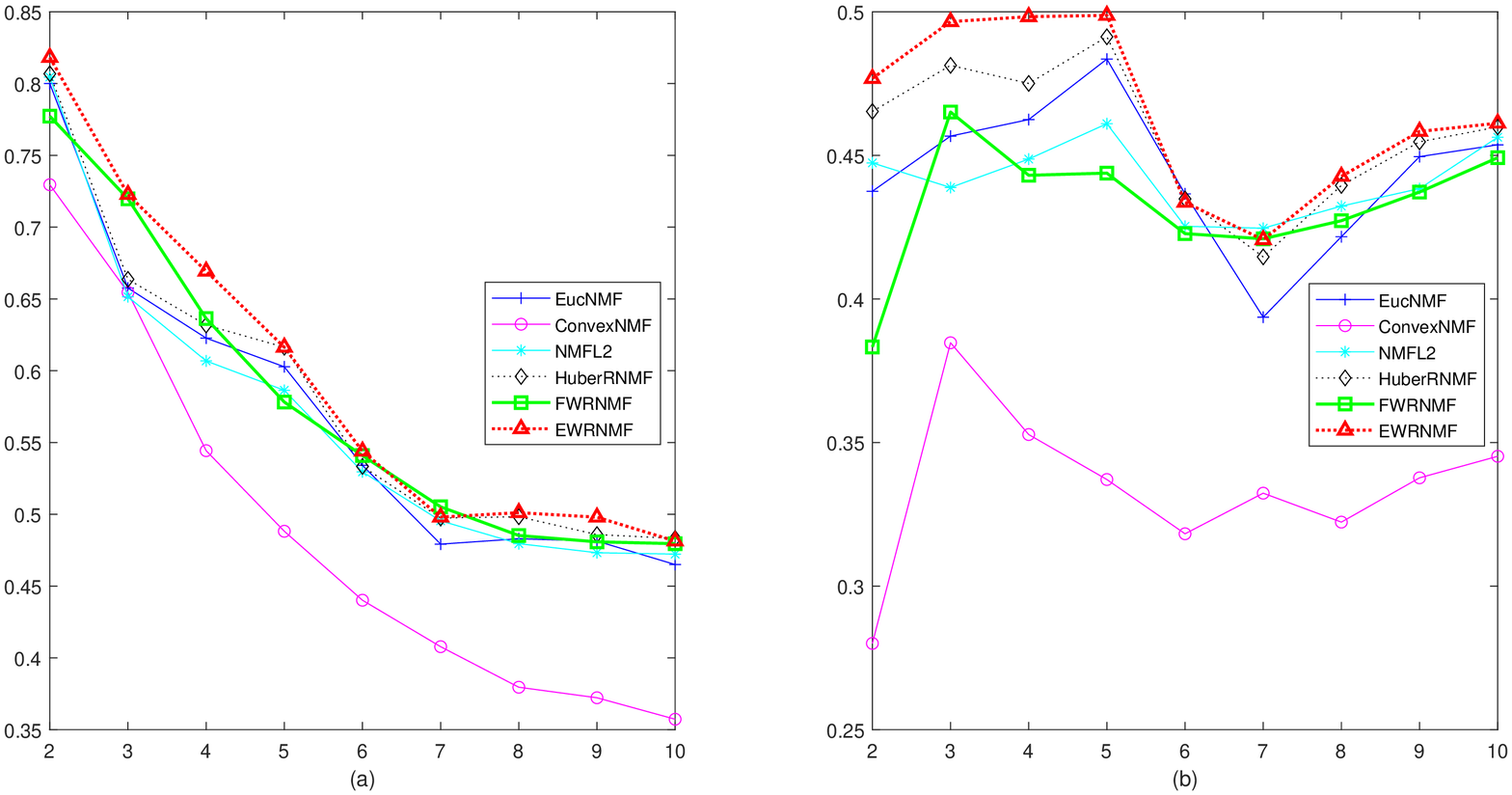}\label{4}
	\caption{Clustering performance versus cluster number on the Yale dataset: (a) ACC and (b) NMI.}
	\label{fig6}
\end{figure*}

\begin{figure*}[htbp]
	\centering
	\includegraphics[height=165pt,width=400pt]{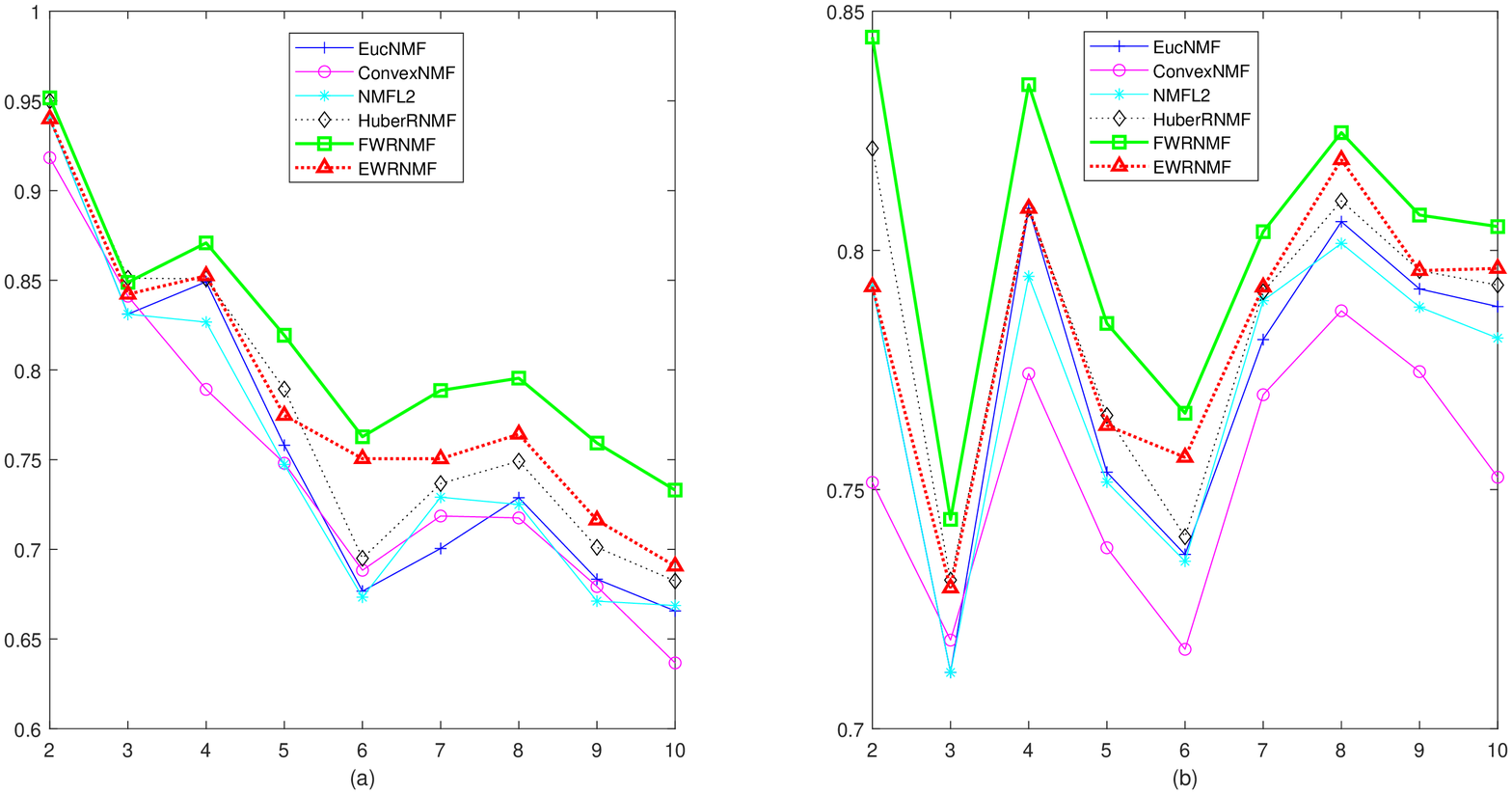}\label{4}
	\caption{Clustering performance versus cluster number on the GTFD dataset: (a) ACC and (b) NMI.}
	\label{fig7}
\end{figure*}


\subsection{Visualization}
We will directly visualize the results of the NMF methods for qualitative comparison. The Dimdata dataset is selected, in which we reduce the dimension number of the original data to two dimensions since it includes two clusters. The separability of new representation ($H$) can be directly observed in Figure \ref{fig9}. By the naked eye, it is easy to observe that the new representations obtained by FWRNMF and EWRNMF have good separability if considering a center-distance-based clustering method such as k-means. This explains why they obtain an extremely excellent result in Tables \ref{tab2} and \ref{tab3} regardless of accuracy or NMI.

\begin{figure*}[!h]
	\centering
	\includegraphics[height=210pt,width=400pt]{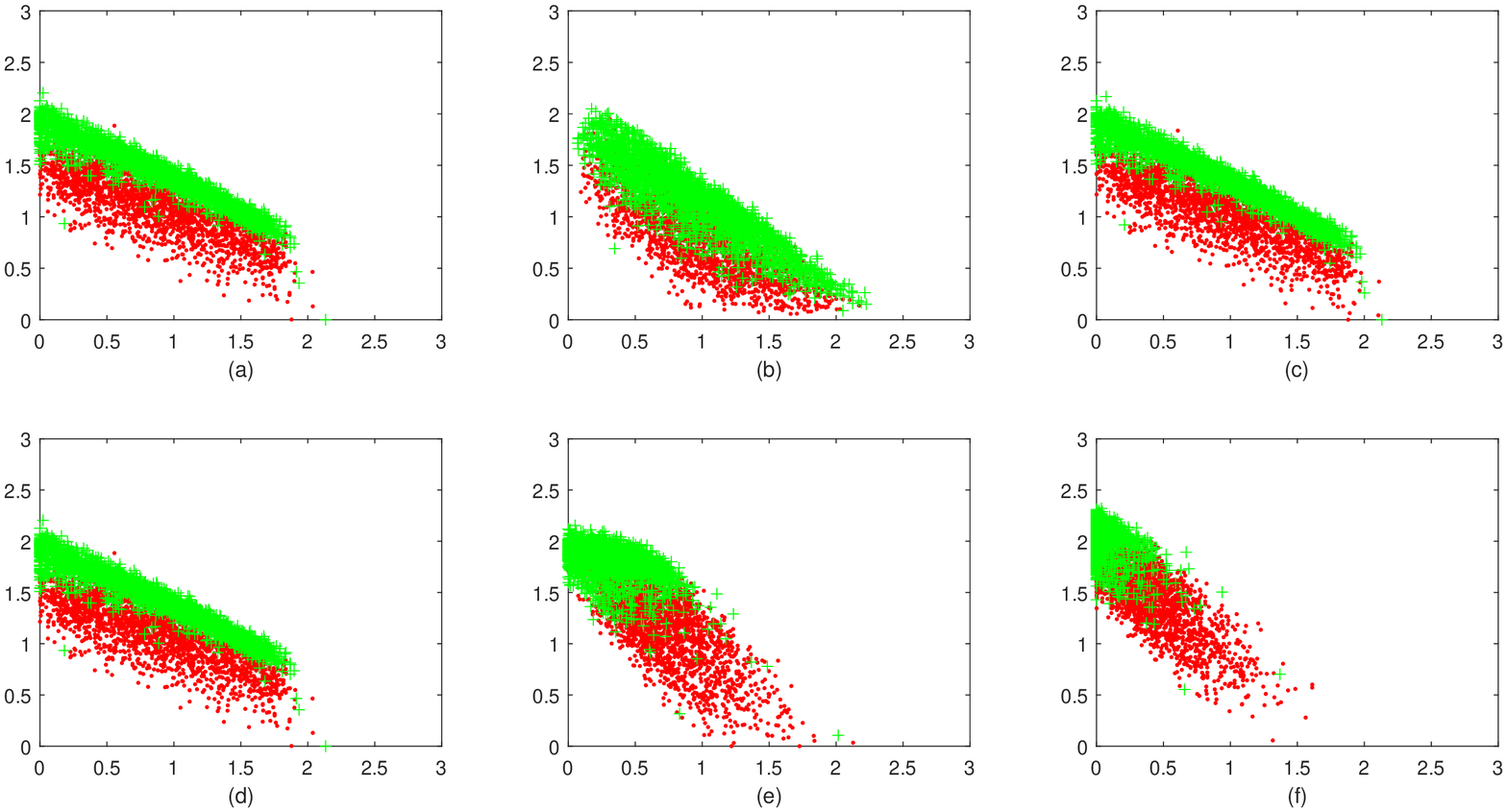}\label{4}
	\caption{Visualization of different methods on the Dimdata dataset which are obtained by (a) EucNMF, (b) ConvexNMF, (c) NMFL2, (d) HuberRNMF, (e) FWRNMF and (f) EWRNMF.}
	\label{fig9}
\end{figure*}
\begin{figure*}[!h]
	\centering
	\includegraphics[height=210pt,width=400pt]{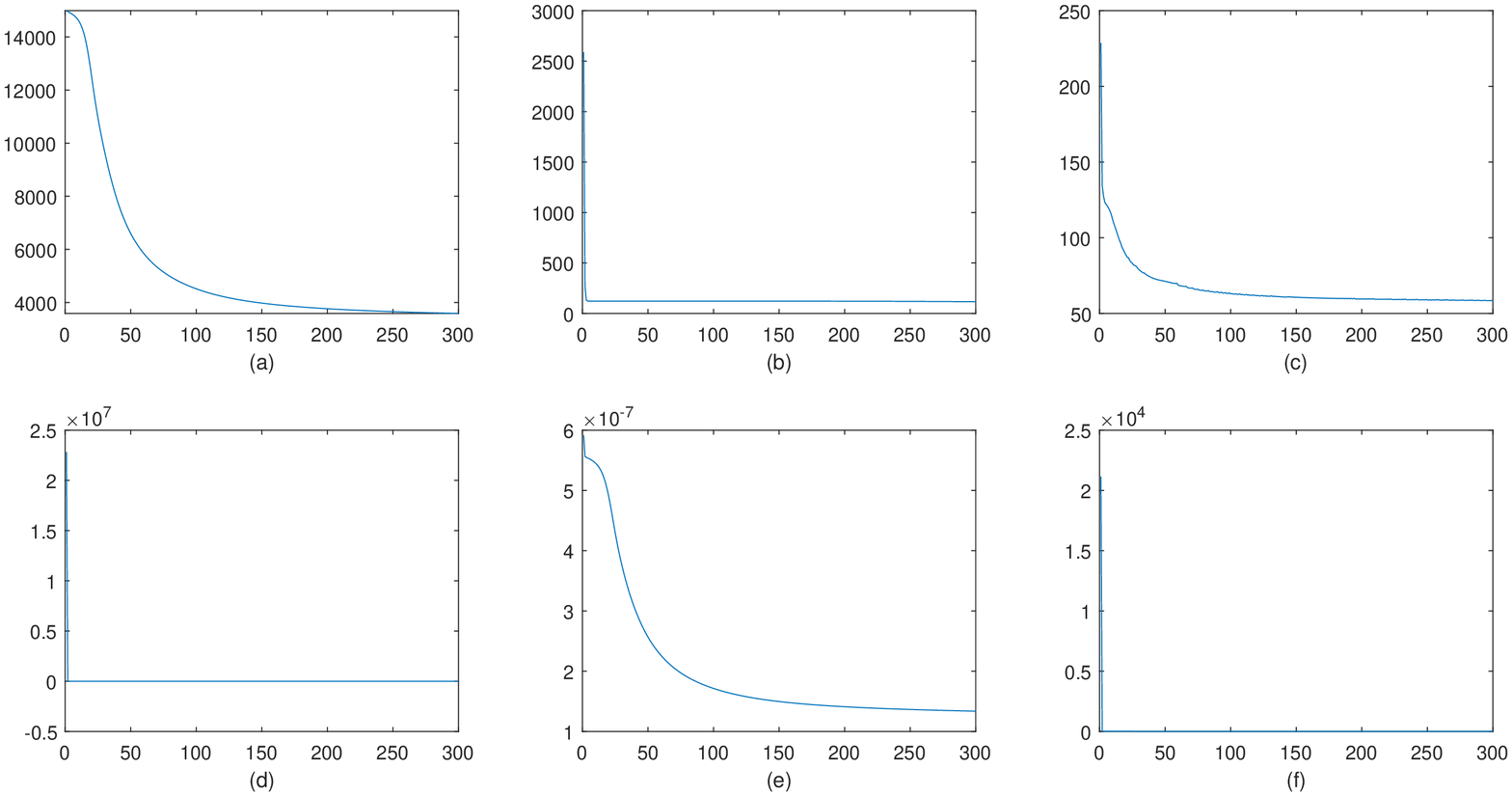}\label{4}
	\caption{Objective function on the ORL dataset which are obtained by (a) EucNMF, (b) ConvexNMF, (c) NMFL2, (d) HuberRNMF, (e) FWRNMF and (f) EWRNMF.}
	\label{fig10}
\end{figure*}

\subsection{Convergence speed and part-based learning}
From the theoretical analysis, we know that the objective function of the proposed methods are monotonically decreasing, but due to the nonconvexity of the objective function, it cannot be guaranteed to be strictly convergent, which is the commonplace in character for NMF methods \cite{chih-jen_lin_convergence_2007}. Thus, we investigate the convergence speed of the NMF methods. Figure \ref{fig10} shows the objective functions on the ORL dataset which are obtained by EucNMF, ConvexNMF, NMFL2, HuberRNMF, FWRNMF and EWRNMF. The convergence speed of the proposed FWRNMF is relatively slow, but EWRNMF shows an excellent performance. However, both methods basically show convergence when iterating 200 times. We also show the base images in Figure \ref{fig11}. As can be seen, ConvexNMF identifies more global faces; and the other methods present similar sparse base images that are difficult to distinguish from each other.

\begin{figure*}[htbp]
	\centering
	\includegraphics[height=190pt,width=500pt]{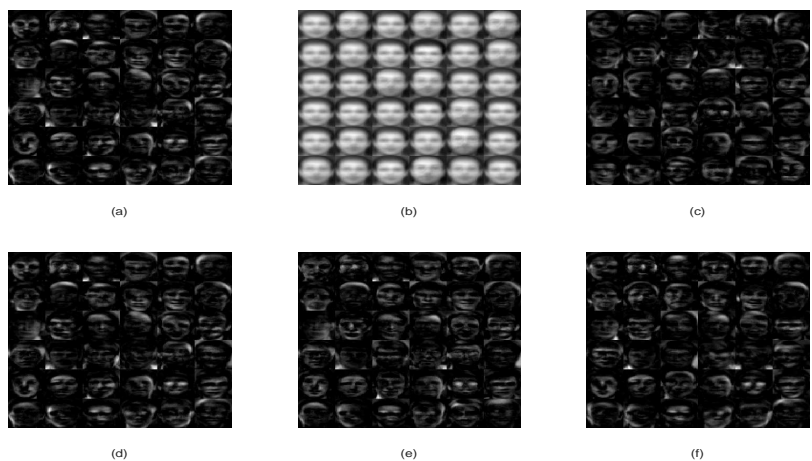}\label{4}
	\caption{Some base images with respect to the ORL dataset which are obtained by (a) EucNMF, (b) ConvexNMF, (c) NMFL2, (d) HuberRNMF, (e) FWRNMF and (f) EWRNMF.}
	\label{fig11}
\end{figure*}

\section{Conclusion}
This paper proposes two new methods to obtain robust feature representation for the NMF problem by introducing two types of adaptive weights for each sample. It is very different from the existing robust NMF methods that use a slow growth measure. These adaptive weights can automatically identify outliers and normal points, and give different importance based on outlier degree. The experimental results show that the proposed methods are flexible and effective.

However, both of new methods require an additional hyperparameter to smooth the weights. In the future, we will develop an automatic method to determine the hyperparameters. 

\section*{Acknowledgement}
This work was supported by  the Fundamental Research Funds for the Central Universities of China (N180719020).


\bibliographystyle{model1-num-names}
\bibliography{references}


\end{document}